\documentclass{bmvc2k}
\usepackage{import}
\usepackage[utf8]{inputenc} % allow utf-8 input
\usepackage[T1]{fontenc}    % use 8-bit T1 fonts
\usepackage{lmodern}

\DeclareGraphicsExtensions{.pdf,.jpeg,.png} % and their extensions so you won't have to specify these with every instance of \includegraphics
\usepackage{wrapfig}
\usepackage{tabularx}
\usepackage{booktabs, multicol, multirow}	% professional-quality tables
\usepackage{array}
\usepackage{algorithmicx}
\usepackage[noend]{algpseudocode}
\usepackage{eqparbox}
\usepackage{algorithm}
\usepackage{amsmath}
\usepackage{mathtools}
\usepackage{commath}
\usepackage{mathrsfs}
\usepackage{amssymb}
\usepackage{amsfonts}       % blackboard math symbols
\usepackage{bbm}
\usepackage{bm}
\usepackage{fancyhdr}
\usepackage{cancel}
\usepackage{nicefrac}       % compact symbols for 1/2, etc.
\usepackage{microtype}      % microtypography
\usepackage{mdwmath}
\usepackage{mdwtab}
\usepackage{wasysym} % for box symbol

\usepackage[capitalize]{cleveref}
\crefname{section}{\S}{\S\S}
\Crefname{section}{Section}{Sections}
\Crefname{table}{Table}{Tables}
\crefname{table}{Tab.}{Tabs.}

\usepackage{url}            % simple URL typesetting

\DeclareMathOperator*{\argmax}{arg\,max}

\newcommand{\E}[2]{\ensuremath{\operatorname{\mathbb{E}}_{#2}\big[ {#1} \big]}}

\newcommand{\R}{\ensuremath{\mathbb{R}}}

\DeclareMathSymbol{\shortminus}{\mathbin}{AMSa}{"39}

\newcommand{\sxtimes}{\mathsf{x}}%\mskip1mu
\newcommand{\T}{\ensuremath{^\intercal}}
\newcommand{\ie}{\textit{i}.\textit{e}., }
\newcommand{\eg}{\textit{e}.\textit{g}. }

\newenvironment{proof}[1]{\par\noindent\underline{Proof:}\space#1}{\hfill $\blacksquare$}

\newtheorem{proposition}{Proposition}[section]

\numberwithin{equation}{section}

\makeatletter
\def\BState{\State\hskip-\ALG@thistlm}
\makeatother

\title{Feature Embedding by Template Matching as a ResNet Block}

\runninghead{Accepted as a conference paper at}{BMVC 2022}

\def\ada{Ada~G\"{o}rg\"{u}n}
\def\yeti{Yeti~Z.~G\"{u}rb\"{u}z}

\def\aydin{A.~Ayd{\i}n~Alatan}
\def\inst1{Dept. of Electrical and Electronics Eng., Middle East Technical University, Ankara, Turkey}

\addauthor{\ada}{ada.gorgun@metu.edu.tr}{1}
\addauthor{\yeti}{yeti@metu.edu.tr}{1}
\addauthor{\aydin}{alatan@metu.edu.tr}{1}

\addinstitution{
 Dept. of Electrical and Electronics Eng. \& Center for Image Analysis (OGAM)\\
 Middle East Technical University\\
 Ankara, Turkey
}

\begin{document}

\maketitle

\begin{abstract}
Convolution blocks serve as local feature extractors and are the key to success of the neural networks. To make local semantic feature embedding rather explicit, we reformulate convolution blocks as feature selection according to the best matching kernel. In this manner, we show that typical ResNet blocks indeed perform local feature embedding via template matching once batch normalization (BN) followed by a rectified linear unit (ReLU) is interpreted as arg-max optimizer. Following this perspective, we tailor a residual block that explicitly forces semantically meaningful local feature embedding through using label information. Specifically, we assign a feature vector to each local region according to the classes that the corresponding region matches. We evaluate our method on three popular benchmark datasets with several architectures for image classification and consistently show that our approach substantially improves the performance of the baseline architectures.
\end{abstract}

%\begin{keyword}
%ResNet \sep arg-max \sep feature embedding \sep template matching \sep classification
%\end{keyword}

%\begin{keywords}
%ResNet, arg-max, embedding
%\end{keywords}

% Paper Part
% ----------

% Introduction
% -------------
\section{Introduction}
\label{sec:introduction}
% Below is an example of how to insert images. Delete the ``\vspace'' line,
% uncomment the preceding line ``\centerline...'' and replace ``imageX.ps''
% with a suitable PostScript file name.
% -------------------------------------------------------------------------
\begin{wrapfigure}[13]{r}[-5pt]{0.485\linewidth} % 15
\vspace{-3.2\intextsep} %\vspace{-2.5\intextsep}
  %\centering
  \centerline{\includegraphics[width=\linewidth]{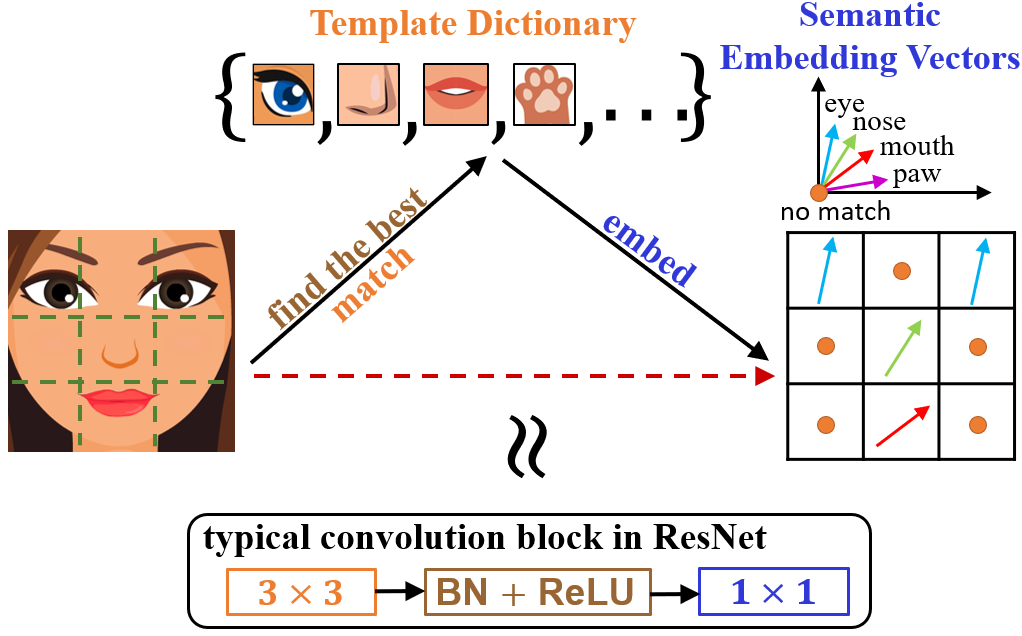}}
  \caption{Visualization of our formulation for local feature embedding and its relation to the typical convolution block existing in ResNet.}
	\label{fig:the_figure_1}
\end{wrapfigure}

Convolutional neural networks (CNN), especially ResNet-like \cite{he2016identity} architectures \cite{resnext,huang2017densely,szegedy2017inceptionv4,zagoruyko2016wide}, are state-of-the-art in image recognition until very recently \cite{chen2021vision}. The success of CNNs heavily relies on hierarchical feature extraction \cite{zhou2018interpreting} through stacked convolution blocks (\ie convolution followed by activation functions) whose parameters are learned in \textit{top-down} manner (\ie via feedback from class-supervised loss function). A possible explanation for the effectiveness of hierarchical feature extraction is considering each pixel in an intermediate feature map as a feature vector corresponding to a semantic entity whose existence with other such features forms some other semantic entities in the successors of the hierarchy (\eg \emph{wing} and \emph{beak} $\to$ \emph{bird}).

Although this folklore is empirically studied in \cite[and references therein]{zhou2018interpreting,zeiler2014visualizing} and further verified for attribute-based zero-shot classification in \cite{demirel2017attributes2classname,xu2020attribute}, its algorithmic implications for \textit{bottom-up} feature extraction are not clear. Thus, the advances typically focus on architectural designs \cite{zagoruyko2016wide,resnext,huang2017densely,szegedy2017inceptionv4} and leave bottom-up feature embedding formulation rather implicit, which might be a lost opportunity in improving the classification performance. Granted that top-down class-supervised feedback is able to shape the bottom-up behaviour through convolutional layers, can we make hierarchical feature extraction more explicit by exploiting supervision in lower levels?

If we were given localized annotations for lower level features in addition to the class labels, all we need would be a bottom-up feature extraction formulation to exploit such supervision. Nevertheless, we do not have such annotations in practice, which makes explicit supervision of intermediate layers a challenge. That being said, it is shown in text domain \cite{word2vec} that linear combination of the vectors corresponding to semantic entities yields the vector of another entity (\eg \emph{woman} $+$ \emph{royal} $\approx$ \emph{queen}). Then the question is \emph{``Can we use mixture of class labels to supervise lower level feature extraction?"}

In this study, we address the challenge of using class-level supervision to explicitly shape the behavior of the intermediate features, which differs from building classifiers at the output of intermediate layers to alleviate vanishing gradient as in GoogLeNet \cite{szegedy2015going} like architectures. We first consider bottom-up formulation of feature embedding through template matching and rigorously show its resemblance to how typical ResNet convolution blocks operate (\cref{fig:the_figure_1}). Building on such a relation, we propose a residual block that assigns a feature vector to each local region according to the classes that the corresponding region matches. We define \emph{best-matching} as a solution of an optimization problem and employ a \emph{soft-max} solution for not only enabling learning but also yielding novel semantic entities as the convex combination of the class features. Specifically, our block is trained with class-level supervision and each local region is encouraged to predict the class of the image it belongs. Surely, some regions are expected to match multiple classes since local features are shared among the classes (\eg \emph{wing} can exist in both \emph{plane} and \emph{bird}). Our method exploits such information to assign semantically meaningful embedding vectors to those regions by combining the vectors of the matched classes. Namely, we explicitly shape the bottom-up behavior of CNNs by learning to combine existing classes to make up new classes for the local regions. We validate our theoretical claims and show the effectiveness of our method with extensive evaluations on 3 popular classification benchmarks.

\section{Related Work}
We discuss the works that are most related to ours. Briefly, our contributions include that $i)$ we re-formulate ResNet block as a feature embedding by template matching, $ii)$ we introduce a batch-statistics-free replacement of BN+ReLU, $iii)$ we develop a residual block that effectively combines the embedding vector of the existing classes to yield embedding vectors to different semantic entities.

Related to interpretive feature embedding, bag of visual words based feature aggregation \cite{arandjelovic2016netvlad} and matching \cite{yeti_bovw} formulations are revisited for global representations. Those approaches build on feature embedding at the top level of CNN's feature extraction hierarchy. On the contrary, our approach explicitly exploits top-down information in earlier stages of the feature extractors for learning their parameters. That being said, our block employs auxiliary classification loss during training similar to \textit{deeply-supervised nets} \cite{lee2015deeply, szegedy2015going}. Those methods employ such loss only in training phase to regularize the features and to facilitate learning without vanishing gradients. Differently, we explicitly use predictions in both training and inference to semantically represent local regions with the combination of class specific vectors, which is a novel approach to use auxiliary loss in intermediate layers.

Our work is mostly related to approaches that are inspired from \textit{attention mechanism} \cite{vaswani2017attention} of natural language processing to express a token in terms of aggregated features within its context. Interpreting convolution as weighted aggregation of local features, predecessors \cite{wang2018non,ramachandran2019stand,hu2019local} replace convolution operation entirely with self-attention for bottom-up design of feature extraction. Albeit self-attention is later proven to express any convolutional layer \cite{cordonnier2019relationship}, patch-matching based vision transformer (\textit{ViT}) \cite{dosovitskiy2020image} shows no such convolution-mimicking attention layer is essential for powerfully expressive models. In our work, our template matching based formulation is also aligned with attention mechanism. Our work differs in that we arrive at similarity-weighted feature aggregation from formally defining the feature embedding through an optimization problem.

As a byproduct connection, activity normalization methods are related to our technique as well. As the pioneer, batch normalization (BN) \cite{normalization2015accelerating} addresses \emph{internal covariate shift} phenomenon. Our theoretical results show that BN has an alternative purpose in BN-ReLU context as \emph{pseudo arg-max} optimizer. Such a relation suggests \emph{margin augmented soft-max}\footnote{A constant is concatenated to the input vector of soft-max.} as an alternative replacement of BN-ReLU to the existing approaches \cite{ulyanov2016instance,ba2016layer,wu2018group} proposed for the relatively small mini-batches.

% METHOD
% -------------
\section{Method}
We repurpose \emph{residual blocks} of a typical residual network \cite{he2016identity} as \emph{feature embedding by template matching} and accordingly, propose a novel residual block (depicted in \cref{fig:method}) that effectively learns local feature embedding from class labels.

We first re-formulate convolution block based local feature embedding as feature assignment through best matching kernel. Relating BN-ReLU to arg-max optimizer, we show that the convolution block of $3{\sxtimes}3$-BN-ReLU-$1{\sxtimes}1$ inherently performs local feature embedding via selecting the best matching convolution kernel (\cref{fig:the_figure_1}). Hence, inspiring from feature embedding by kernel matching interpretation, we develop our residual block.  

\subsection{Feature Embedding by Template Matching}
\label{sec:formulation}
We are given a feature map, $f\in\R^{w{\sxtimes}h{\sxtimes}d}$, which is the output of some NN layer. At each spatial location (\ie pixel), we have a feature $x\in\R^{d}$ that possibly represents a local region around it to some spatial extent.

We want to obtain a feature map, $f^{\prime}\in\R^{w^{\prime}{\sxtimes}h^{\prime}{\sxtimes}d^{\prime}}$, from $f$ by transforming $x$ into another vector that captures the semantics of local neighborhood. We let $x_{3{\sxtimes}3}\in\R^{9d}$ denote concatenated features of $3{\sxtimes}3$ window centered at $x$. We have a set of matching kernels $\{\omega_k\in\R^{9d}\}_k$ each of which seeks for a particular pattern. To each kernel $\omega_k$, we associate an embedding vector, $\nu_k\in\R^{d^\prime}$, representing the semantics of the corresponding $3{\sxtimes}3$ pattern. We aim to replace $x$ with the embedding vector of the best matching kernel to its neighborhood. Hence, we formally define the problem as:
\begin{equation}\tag{P1}\label{eq:argmax}
p^\ast =  \argmax_{\substack{p,q\geqslant0\\q+\Sigma_k p_k =1}} q\,\mu+\textstyle\sum_k p_k\, \omega_k\T x_{3{\sxtimes3}}
\end{equation}
where $\mu$ is a threshold to zero out the embedding vector when no kernel is  matched with at least $\mu$ similarity. $p^{\ast}$ is either one-hot or zero vector owing to \textit{total unimodularity} \cite{unimodular} of the constraints. We have $p^{\ast} = 0$ when any of the activations, $a_k=\omega_k\T x_{3{\sxtimes}3}$, are no greater than $\mu$. Then, we obtain the representation of $x$ as $x^{\prime} = \Sigma_k p^\ast_k\,\nu_k$.

Given the initial feature map, $f$, the transformed feature map, $f^{\prime}$, can be efficiently obtained by $3{\sxtimes}3$ convolution with kernels $\{\omega_k\}_k$, solving a linear program and $1{\sxtimes}1$ convolution with vectors $\{\nu_k\}_k$, sequentially. Although computationally efficient, one critical problem with such formulation is that the linear program breaks the back-propagation of the computational graph. Namely, $p^\ast$ as a function of $a$ is not smooth where $a_k=\omega_k\T x_{3{\sxtimes3}}$.

To alleviate non-differentiability of the linear program, we can use stochastically perturbed optimizers \cite{berthet2020learning}:
\begin{equation}\tag{P2}\label{eq:perturbed_argmax}
p^\ast \!=  \E{\!\argmax_{\substack{p,q\geqslant0\\q+\Sigma_k p_k =1}} \!q\,(\mu{+}\tfrac{1}{\epsilon}z^\prime)+p\T(a{+}\tfrac{1}{\epsilon}z)}{z^\prime\!,z{\sim}\mathcal{N}(0,I)}
\end{equation}
or we can use entropy regularization to make the problem strictly concave and smooth:
\begin{equation}\tag{P3}\label{eq:entropy_argmax}
\begin{split}
p^\ast = \argmax_{\substack{p,q\geqslant0\\q+\Sigma_k p_k =1}} q\,\mu+ p\T a -\tfrac{1}{\epsilon}(q\log q+p\T\log p)
\end{split}
\end{equation}
where $\epsilon$ in both problems controls how smooth the solution $p^\ast$ is to be. We will introduce two propositions that ensure the existence of the \textit{Jacobian} $[\tfrac{\partial p^\ast}{\partial a}]_{ij}\coloneqq\tfrac{\partial p^\ast_j}{\partial a_i}$.

\begin{proposition}[follows\! from\! Lemma\! 1.5\! \cite{abernethy2016perturbation}] \label{proposition_perturbed}
Given samples $z^\prime,z$ from standard normal distribution, let $\Tilde{p}(z^\prime, z)\coloneqq\argmax\limits_{\substack{p,q\geqslant0\\q+\Sigma_k p_k =1}} \!q\,(\mu{+}\tfrac{1}{\epsilon}z^\prime)+p\T(a{+}\tfrac{1}{\epsilon}z)$. If $p^\ast$ is the solution of the problem \eqref{eq:perturbed_argmax}, then we have:
$$\tfrac{\partial p^\ast}{\partial a}=\E{\epsilon\Tilde{p}(z^\prime, z)\,z\T}{z^\prime,z\sim\mathcal{N}(0,I)}$$
\end{proposition}

\begin{proposition} \label{proposition_entropy}
The solution of the problem \eqref{eq:entropy_argmax} admits closed form expression as $p^\ast_k = \tfrac{\exp(\epsilon a_k)}{\exp(\epsilon \mu)+\Sigma_{k^\prime}\exp(\epsilon a_{k^\prime})}$ (\ie \textit{soft-max}) and we have $\tfrac{\partial p^\ast}{\partial a}=\epsilon(\Lambda(p^\ast) - p^\ast p^{\ast\intercal})$
where $\Lambda(p^\ast)$ is the diagonal matrix with $p^\ast$ as the diagonal.
\end{proposition}
\begin{proof}
The results follow from the first order optimality conditions owing to strict concavity.
\end{proof}

\bigbreak
The two propositions enable us to implement the best matching kernel selection as a differentiable layer using soft maximizers. $p^\ast$ will no longer be a one-hot or zero vector. Granted, the entities of $p^\ast$ will decay to zero if no activation is greater than $\mu$ and we will possibly have multiple non-zero entities otherwise due to soft-max operation. To this end, BN-ReLU can be interpreted as a soft approximation of the problem \eqref{eq:argmax} as we will show shortly.

\subsection{BN-ReLU as a Soft Maximizer of \eqref{eq:argmax}}
\label{sec:bnargmax}

BN \cite{normalization2015accelerating} and its successor counterparts \cite{ulyanov2016instance,ba2016layer,wu2018group} perform activity normalization of the form $\hat{a}_k=\gamma_k\tfrac{a_k-\mathbb{E}[a_k]}{\sqrt{\mathrm{Var}(a_k)}}+\beta_k$ using some batch statistics. Applying ReLU to $\hat{a}$, we obtain $\hat{p}=\max(\hat{a},0)$. Given $\{\nu_k\}_k$ embedding vectors, we compute the output feature as $x^{\prime} = \Sigma_k \hat{p}_k\,\nu_k$. Denoting $\eta\coloneqq\Sigma_k\hat{p}_k$ and $\hat{p}^\ast_k =\nicefrac{\hat{p}_k}{\eta}$, we can equivalently write $x^{\prime} = \eta\Sigma_k \hat{p}^\ast_k\,\nu_k$, where $\hat{p}^\ast$ is a feasible solution for problem $\eqref{eq:argmax}$ and indeed is the optimal solution when all the activations are less than $\mu_k$ for $\mu_k=\mathbb{E}[a_k]-\tfrac{\beta_k\sqrt{\mathrm{Var}(a_k)}}{\gamma_k}$. Moreover, $\hat{p}^\ast$ preserves the relative ordering of the values in the solution of the problem \eqref{eq:entropy_argmax}. In fact, BN maps activations around 0 where we have $\mathrm{e}^x \approx 1 + x$, meaning that BN-ReLU is a biased first order approximation for unnormalized soft-max for the non-negative activations. Hence, BN-ReLU can be interpreted as yielding a scaled soft maximizer to the problem \eqref{eq:argmax}.

We support our claims on such a relation with empirical studies (\cref{sec:ablation}) where we replace BN-ReLU with perturbed maximizer \cite{berthet2020learning} and soft-max layers and scale the output with a constant. Such replacement of BN-ReLU mitigates \textit{batch-statistics} demand in activity normalization.

Showing the approximate equivalence between BN-ReLU and arg-max, we can use convolution block of $3{\sxtimes}3$-BN-ReLU-$1{\sxtimes}1$  to implement our local feature embedding by template matching. In fact, $3{\sxtimes}3$-BN-ReLU-$1{\sxtimes}1$ is a typical block exploited in ResNet based architectures \cite{he2016identity,resnext, huang2017densely}. Thus, our formulation of local feature embedding provides a different insight towards explanation of how ResNets succeed. Besides, our formulation suggests that $3{\sxtimes}3$-BN-ReLU-$1{\sxtimes}1$ convolution block is mimicking \textit{cross-attention} \cite{vaswani2017attention} between $3{\sxtimes}3$ patches and convolution kernels. Namely, $3{\sxtimes}3$ patches are \emph{queries} and the convolution kernels are the \emph{keys}. Each patch is represented by a vector which is the convex combination of \emph{value} vectors corresponding to \emph{keys}.

% Below is an example of how to insert images. Delete the ``\vspace'' line,
% uncomment the preceding line ``\centerline...'' and replace ``imageX.ps''
% with a suitable PostScript file name.
% -------------------------------------------------------------------------
\begin{figure}[!hb]
%\begin{wrapfigure}[16]{r}[-5pt]{0.5\linewidth}
%\vspace{-3\intextsep}
  \centering
  \centerline{\includegraphics[width=1.\linewidth]{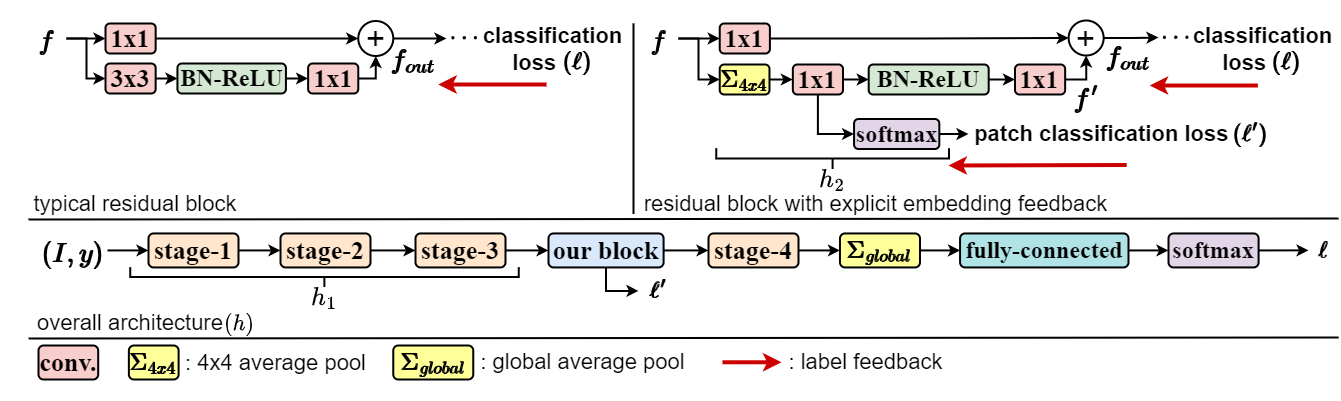}}
  \caption{Computation flow of a residual block (top-left), our feature embedding block (top-right), and the overall architecture equipped with our method (bottom).}
	\label{fig:method}
  %\end{wrapfigure}
%
\end{figure}

\subsection{Explicit Feature Embedding as Residual Block}
\label{sec:proposed_stage}
We just show that bottom-up behavior of CNNs having $3{\sxtimes}3$-BN-ReLU-$1{\sxtimes}1$ blocks within is feature vector assignment by template matching. In particular, the embedding vector of a $3{\sxtimes}3$ patch is the scaled convex combination of the \emph{value} vectors corresponding to the convolution kernels, where combination weights are proportional to the matching scores. Thus, each residual block of ResNet (\cref{fig:method}) can be interpreted as enhancing the feature vectors in the input feature map through shortcut connection by the semantic vectors of the best matching patterns to the corresponding features' $3{\sxtimes}3$ neighborhood. Following this perspective, we now formulate our feature embedding mechanism.

Instead of $3{\sxtimes}3$ windows, we consider a larger spatial extent (\ie \emph{patches}) centered around each pixel in a feature map. Our aim is to match such patches to classes rather than convolution kernels. We achieve this by training an auxiliary classifier for the patches along with the main classifier. Inevitably, the patches having shared entities among classes will not be discriminative enough and will match to multiple classes to minimize the classification loss. We rigorously make use of such behaviour to embed patches regarding their semantic meaning by using learnable embedding vectors, \ie \emph{value} vectors, for the classes. Specifically, we use the prediction scores to compute convex combination of the \emph{value} vectors. Provided that the learned \emph{value} vectors correspond to semantics of the classes, then their combination will correspond to new semantic entities (\eg 0.5\emph{plane} + 0.5\emph{bird} $\approx$ \emph{wing}). In this way, we manage to exploit weighted combination of the labels to explicitly supervise local feature extraction.

Formally, given an input feature map, $f\in\R^{w\sxtimes h\sxtimes d}$, we extract $\tfrac{w}{2}\sxtimes \tfrac{h}{2}$ patches, $x_{\Box}\in\R^{\tfrac{w}{2}\sxtimes \tfrac{h}{2}\sxtimes d}$ where $x$ with box $\Box$ is a patch centered at $x$. We then obtain a global representation by average pooling for each patch as $x_g = \tfrac{1}{\vert x_{\Box}\vert}\Sigma_{x\in x_{\Box}} x$ where $\vert x_{\Box}\vert$ denotes the number of features. We apply a $1\sxtimes 1$ convolution (\ie linear transform) with bias to obtain class matching scores (\ie activations, $a$, in the context of our original formulation in \cref{sec:formulation}) for $c$-many classes as $a_k = \alpha_k\T x_g + \beta_k$ for $k\in[1\ldots c]$ where $\alpha_k$ and $\beta_k$ are the trainable vector and the bias term for class $k$. 

To learn the classifier parameters, $(\alpha,\beta)$, we augment the training loss with an auxiliary per patch classification loss. Hence, we are able to propagate label supervision in different levels to explicitly encourage feature embedding by template matching paradigm. The loss for a dataset, $\mathcal{D}$, of image($I$)-label($y)$ tuples becomes:
\begin{equation}\label{eq:aux_loss}
    \mathcal{L}(\mathcal{D}) = \tfrac{1}{\vert\mathcal{D}\vert}\!\!\!\!\!\!\!\textstyle\sum\limits_{(I,y){\in}\mathcal{D}} \!\!\!\!\!\!\big[(1{\shortminus}\lambda)\ell(h(I),y)+\tfrac{1}{wh}\!\!\!\!\!\!\!\textstyle\sum\limits_{x{\in}h_1(I)}\!\!\!\!\!\!\lambda\ell(h_2(x_\Box),y)\big]
\end{equation}
where $h_1(\cdot)$ denotes the network output of size $w\sxtimes h$ until our layer, $h_2(\cdot)$ denotes our layer's class scores, $h(\cdot)$ denotes the whole network's class scores and $\ell(\cdot)$ is the \textit{cross-entropy} loss of soft-maxed scores.

Finally, following our results from \cref{sec:formulation,sec:bnargmax}, we apply BN-ReLU-$1{\sxtimes}1$ convolution block to obtain the final representation, $x^\prime\in\R^d$, for the patch $x_\Box$. Namely, to each class, we associate an embedding vector, $\nu_k\in\R^d$, to describe the whole patch as $x^\prime=\Sigma_k \hat{p}_k\,\nu_k$ where $\hat{p}$ is the output of BN-ReLU as we explain in \cref{sec:bnargmax}. We should note that we use soft-max in loss computation to have normalized probabilities and we rigorously use BN-ReLU for the mixing coefficients to tackle no-match cases while soft-maxing. Hence, our method matches local regions to the class labels rather than particular patterns and embeds the corresponding semantic information as the scaled convex combination of the class semantics so that the embedded semantic is to be useful in the further levels of the feature embedding hierarchy. Similar to typical residual block, we add the resultant feature map, $f^\prime$, to the initial map, $f$, via shortcut connection with a per-pixel linear transform, \ie $f^{out}=\mathrm{conv}_{1\sxtimes 1}(f) + f^\prime$.

\subsection{Implementation Details}
\label{sec:implementation}

We use ResNet (RN) \cite{he2016identity}, Wide-ResNet (WRN) \cite{zagoruyko2016wide} of \emph{depth} 16 and \emph{widening factor} 2, and DenseNet (DN) \cite{huang2017densely} of \emph{depth} 100 and \emph{growth rate} 12 as the baseline architectures each of which has 4 stages. In RN and WRN, we have spatial reduction in stage-2 and stage-3 whereas in DN, we have spatial reduction in the first two stages. We summarize the general architecture in \cref{fig:method} where we also show our feature embedding mechanism as well as $h_1(\cdot)$ and $h_2(\cdot)$ in \cref{eq:aux_loss}. We place our layer in between the last two stages. We only add an extra classification and two linear transforms (\ie three $1{\sxtimes}1$ convolutions)  to the baselines. For DN, we additionally employ concatenation of $f^\prime$ and $f$ instead of addition through shortcut to align with the architectural design of DN. We provide further details for reproducibility in the supplementary material.

% EXPERIMENTS
% -----------
\section{Experimental Work}
\label{sec:setup}
We evaluate the effectiveness of the proposed feature embedding method for the image recognition task. We further perform ablation studies for the implications of our formulations as well as the effects of the hyperparameters.

\textbf{Datasets.} 100-class Mini-ImageNet \cite{Ravi2017OptimizationAA} with images of size $84{\sxtimes}84$ and Cifar \cite{krizhevsky2009learning} (10 and 100) with images of size $32{\sxtimes}32$. We use splits of 65\%, 15\%, 20\% for train, validation, test sets with train data augmentation of \cite{he2016identity}.

\textbf{Training.} Default \textit{Adam} optimizer with $10^{{\shortminus}3}$ learning rate, $10^{{\shortminus}4}$ weight decay, and mini-batch size of 32.

\textbf{Hyperparameters.} We set $\lambda{=}0.5$ in \cref{eq:aux_loss} based on our ablation study (\cref{fig:lambda_ablation}). Due to larger images of Mini-ImageNet, we  employ additional spatial reduction in the first stages of RN and WRN, and in the third stage of DN to have similar output feature size with Cifar.

\subsection{Ablation Studies}
\label{sec:ablation}

% Below is an example of how to insert images. Delete the ``\vspace'' line,
% uncomment the preceding line ``\centerline...'' and replace ``imageX.ps''
% with a suitable PostScript file name.
% -------------------------------------------------------------------------
\begin{wrapfigure}[10]{r}[-5pt]{0.35\linewidth}
\vspace{-1.1\intextsep}
%\begin{minipage}{.49\linewidth}
%\begin{figure}[!ht]
  %\centering
  \centerline{\includegraphics[width=1.0\linewidth,keepaspectratio]{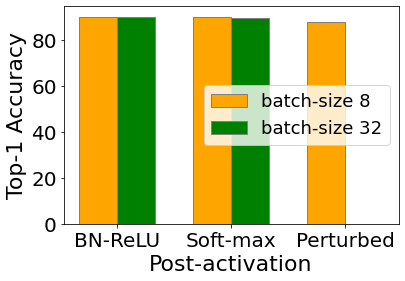}}
  %\smallskip
  %\centerline{(b)}
  %\end{figure}
%\end{minipage}
\caption{Replacing BN-ReLU}
\label{fig:bnrelu}
\end{wrapfigure}

\textbf{Replacing BN-ReLU with soft-maximizers.} To support our claims in \cref{sec:bnargmax}, we replace BN-ReLU following $3{\sxtimes}3$ convolution with perturbed maximizer \cite{berthet2020learning} and soft-max layers with $\mu$ and $\eta$ constants from \cref{sec:bnargmax}. In particular, we concatenate $\mu$ to activations and perform soft-max, which we refer \textit{margin augmented soft-max}. We then scale the output by $\eta$. Using $\mu_k=\mathbb{E}[a_k]-\tfrac{\beta_k\sqrt{\mathrm{Var}(a_k)}}{\gamma_k}$, we estimate $\mu{=}2.5$ from BN layers of a pre-trained ResNet20 as the average of non-zero $\mu_k$ for each activation. Similarly, we use $\eta{=}17$ from the average of per-pixel sum of the activations after BN-ReLU. For perturbed maximizer \cite{berthet2020learning}, we use 600 samples for empirical expectation. We use $\epsilon{=}1$ for both methods based on the ablation study in \cite{berthet2020learning}. We evaluated the methods with relatively small (8) and larger (32) batch sizes except we exclude perturbed maximizer in 32 batch size due to its memory demand. We use 3-stage 2-block ResNet20 \cite{he2016identity} baseline and Cifar-10 dataset in our evaluation. The comparisons are provided in \cref{fig:bnrelu}. We observe that the methods perform on par with each other. Supporting our claims in \cref{sec:bnargmax}, such empirical results also suggest a technique for activity normalization without using batch-statistics.

% Below is an example of how to insert images. Delete the ``\vspace'' line,
% uncomment the preceding line ``\centerline...'' and replace ``imageX.ps''
% with a suitable PostScript file name.
% -------------------------------------------------------------------------
\begin{wrapfigure}[12]{l}[-5pt]{0.37\linewidth}
\vspace{0\intextsep}
%\begin{minipage}{.49\linewidth}
%\begin{figure}[!ht]
  %\centering
  \centerline{\includegraphics[width=1.0\linewidth,keepaspectratio]{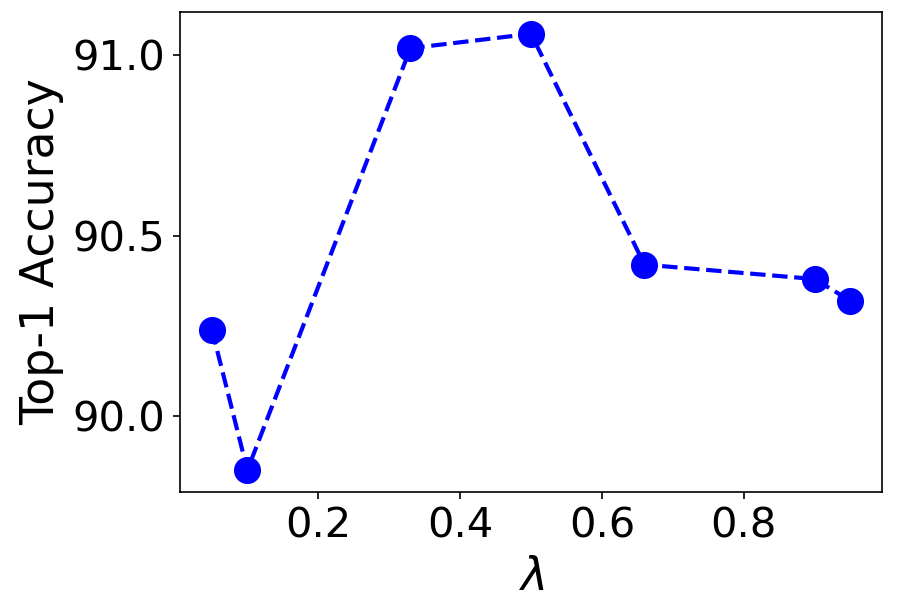}}
  %\smallskip
  %\centerline{(b)}
  %\end{figure}
%\end{minipage}
\caption{Effect of $\lambda$.}
\label{fig:lambda_ablation}
\end{wrapfigure} \textbf{Effect of $\mathbf{\lambda}$}. We perform grid search on $\lambda$ mixing coefficient for the two losses in \cref{eq:aux_loss} (\ie $\ell$ and $\ell^\prime$ in \cref{fig:method}). We use 4-stage 2-block ResNet with our method and Cifar-10 dataset in our evaluation. We provide the results in \cref{fig:lambda_ablation}. Small $\lambda$ values (\ie absence of auxiliary loss) degrades the performance. We find that equally weighting the losses ($\lambda{=}0.5$) brings the best performance. 

\textbf{Number of blocks (\emph{depth}).} We evaluate both 2-block (RN26) and 3-block (RN38) stages in RN baseline to examine the effect of our feature embedding with the increased depth. The comparisons are provided in \Cref{tab:alldata} where we observe that increased depth boosts the performance of our method. Notably, we also observe that our method with less depth performs on par with the baseline of more depth.

\subsection{Classification Results}
% \usepackage{graphicx}

\begin{wraptable}[13]{r}[-5pt]{0.6\linewidth}
    \vspace{-.9\intextsep}
	\centering
	\caption{Evaluation on image recognition task. Bold: best in its category. $C:$ the number of classes.}
	\label{tab:alldata}
	\resizebox{\linewidth}{!}{%
		\begin{tabular}{lcccc}
			\toprule 
			Dataset $\rightarrow$ &\multicolumn{1}{c}{}& \multicolumn{1}{c|}{Cifar10} & \multicolumn{1}{c|}{Cifar100} & \multicolumn{1}{c}{Mini-ImageNet} \\ 
			\cmidrule{3-3} \cmidrule{4-4} \cmidrule{5-5}
			Architecture $\downarrow$ & Params & top-1 acc. & top-1 acc. & top-1 acc.\\ \midrule
			RN26 & {\color[HTML]{000000} 0.96M$+257C$} & {\color[HTML]{000000} 89.52} & {\color[HTML]{000000} 65.94} & {\color[HTML]{000000} 60.43}  \\
			RN26-aux. & {\color[HTML]{000000} 0.96M$+386C$} & {\color[HTML]{000000} 90.57} & {\color[HTML]{000000} 66.21} & {\color[HTML]{000000} 60.70} \\
			RN26-Ours & {\color[HTML]{000000} 0.98M$+516C$} & {\color[HTML]{000000} \textbf{91.06}} & {\color[HTML]{000000} \textbf{66.78}} & {\color[HTML]{000000} \textbf{61.23}}\\ \midrule
			RN38 & {\color[HTML]{000000} 1.42M$+257C$} & {\color[HTML]{000000} 90.78} & {\color[HTML]{000000} 68.15} & {\color[HTML]{000000} 60.72}\\ 
			RN38-Ours & {\color[HTML]{000000} 1.44M$+516C$} & {\color[HTML]{000000} \textbf{91.36}} & {\color[HTML]{000000} \textbf{69.01}} & {\color[HTML]{000000} \textbf{63.83}}\\ \midrule
			WRN16 & {\color[HTML]{000000} 1.28M$+129C$} & {\color[HTML]{000000} 90.52} & {\color[HTML]{000000} 67.11} & {\color[HTML]{000000} 60.73}\\ 
			WRN16-Ours & {\color[HTML]{000000} 1.30M$+388C$} & {\color[HTML]{000000} \textbf{91.10}} & {\color[HTML]{000000} \textbf{67.36}} & {\color[HTML]{000000} \textbf{62.92}}\\ \midrule
			DN100 & {\color[HTML]{000000} 1.20M$+535C$} & {\color[HTML]{000000} 92.62} & {\color[HTML]{000000} 71.65} & {\color[HTML]{000000} 65.03}\\ 
			DN100-Ours & {\color[HTML]{000000} 1.32M$+1222C$} & {\color[HTML]{000000} \textbf{92.92}} & {\color[HTML]{000000} 71.25} & {\color[HTML]{000000} 68.86}\\
			DN100-Ours-C & {\color[HTML]{000000} 1.36M$+1264C$} & {\color[HTML]{000000} 92.71} & {\color[HTML]{000000} \textbf{72.14}} & {\color[HTML]{000000} \textbf{68.93}}\\\bottomrule
		\end{tabular}%
	}  
\end{wraptable}

%Comparison with three popular image classification benchmarks CIFAR10, CIFAR100 and Mini-ImageNet for the baseline and the proposed models, including top-1 accuracy and the number of parameters. Bold: best in its category. $C:$ the number of classes.
We train several architectures (RN\#, WRN16, DN100) equipped with our feature embedding block (\emph{Baseline}-Ours). The baselines are of different architectural choices with varying depths. Our aim is rather to show the effectiveness of our theoretical derivations than to push state-of-the-art (SOTA) by architecture design. We firmly believe that our experiments are sufficient to validate the effectiveness and the generalization capability of our method as well as our claims. 

In order to minimize the confounding of the factors other than our proposed method, we keep the comparisons as fair as possible following the same experimental settings disclosed in \cref{sec:setup} for all architectures. We provide the results in \Cref{tab:alldata} where we mark all results that outperform its baseline counterpart. We observe that we improve the performance of WRN and DN, which are SOTA CNN architectures. Moreover, our method consistently improves all the baselines and predominantly, such improvement does not come from the marginal parameter increase that our method brings. 2-block RN26 with our method is mostly superior to its 3-block baseline (RN38). In relatively shallow architectures, our method's improvement is more significant. With DN architecture, we also experiment enhancing the features by concatenation (DN-Ours-C) instead of addition (\cref{sec:implementation}). Concatenation is marginally superior to addition in DN owing to better alignment with the architecture of DN.

We also evaluate RN26 with auxiliary classification loss only as in \cite{lee2015deeply, szegedy2015going} to show the efficiency of our contribution which is exploiting matching scores as the mixing coefficients for the class embedding vectors. Our method brings consistent improvements in all datasets with respect to direct application of auxiliary classification loss in the intermediate layers.

\subsection{Analysis of Feature Embedding Behaviour}
We further analyze the effect of our feature embedding mechanism with RN26 in Cifar10 dataset through t-SNE plots of the features (\cref{fig:geometry_comparison,fig:embedding_geometry}) as well as sample patches (\cref{fig:words}) corresponding to spatial extent of the features. We sample 80 images for each class and project the pixels at the feature maps to 2D space.

% Below is an example of how to insert images. Delete the ``\vspace'' line,
% uncomment the preceding line ``\centerline...'' and replace ``imageX.ps''
% with a suitable PostScript file name.
% -------------------------------------------------------------------------
\begin{figure}[!ht]%[15]{r}[-5pt]{0.8\linewidth}
%\vspace{-1\intextsep}
  \centering
  \centerline{\includegraphics[width=1.0\linewidth]{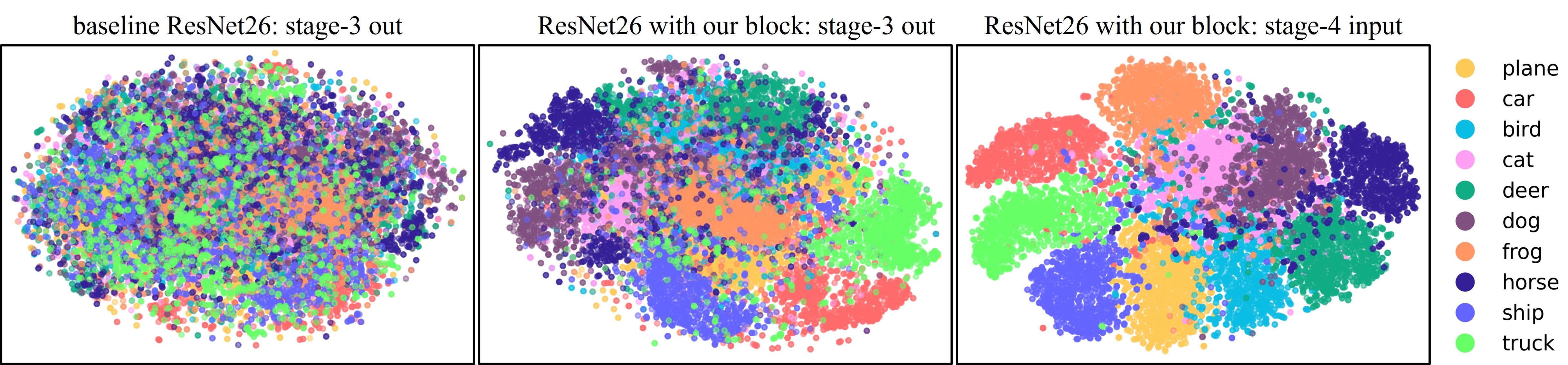}}
  \caption{2D t-SNE projections of the features with and without our method.}
	\label{fig:geometry_comparison}
  \end{figure}

%
%\end{wrapfigure}

% Below is an example of how to insert images. Delete the ``\vspace'' line,
% uncomment the preceding line ``\centerline...'' and replace ``imageX.ps''
% with a suitable PostScript file name.
% -------------------------------------------------------------------------
\begin{wrapfigure}[14]{r}[-5pt]{0.5\linewidth}
%\vspace{-3\intextsep}
\centerline{
\begin{minipage}[b]{.42\linewidth}
  \centering
  \centerline{\includegraphics[width=\linewidth]{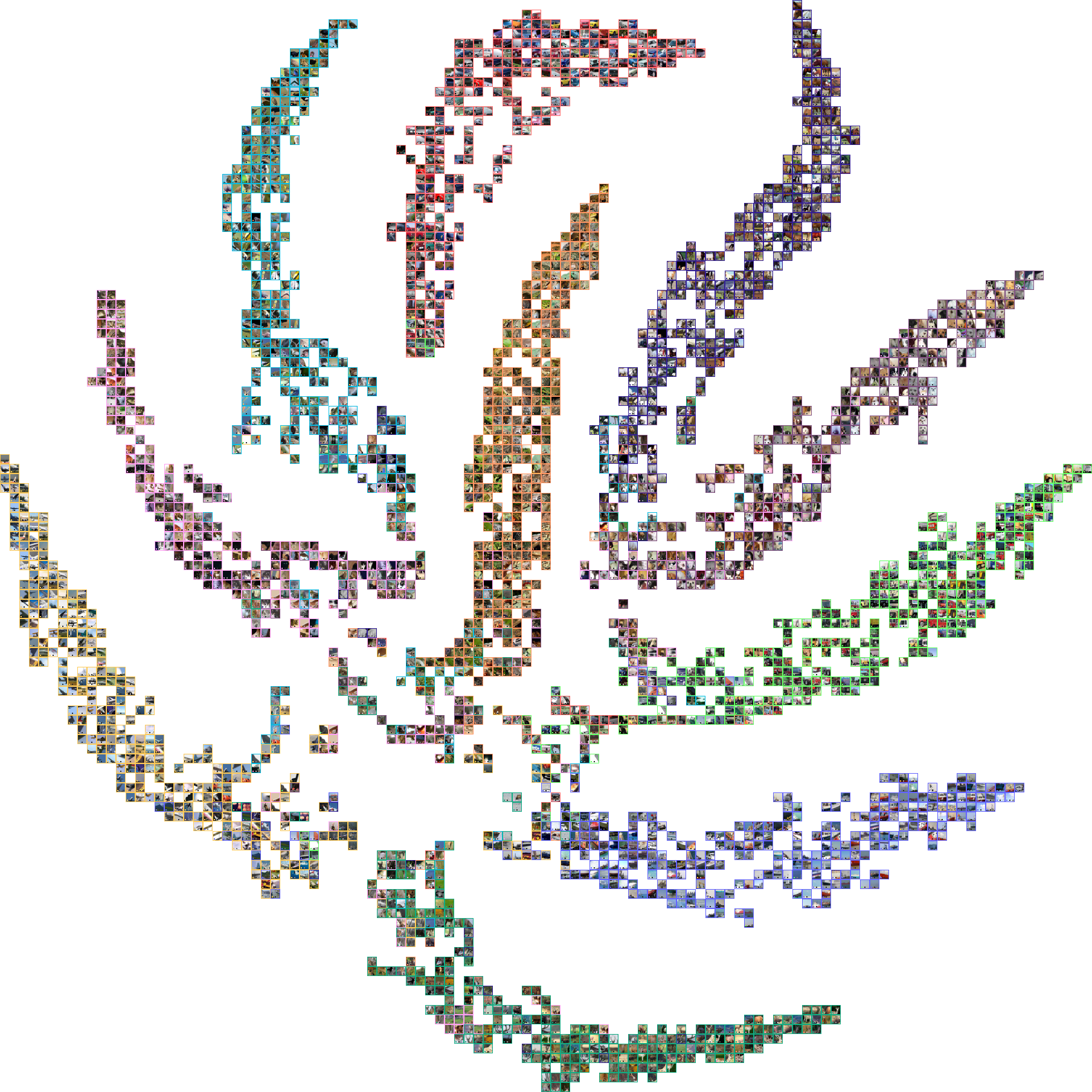}}
%  \vspace{2.0cm}
  %\centerline{\small{softmax}}
\end{minipage}
\begin{minipage}[b]{.42\linewidth}
  \centering
  \centerline{\includegraphics[width=\linewidth]{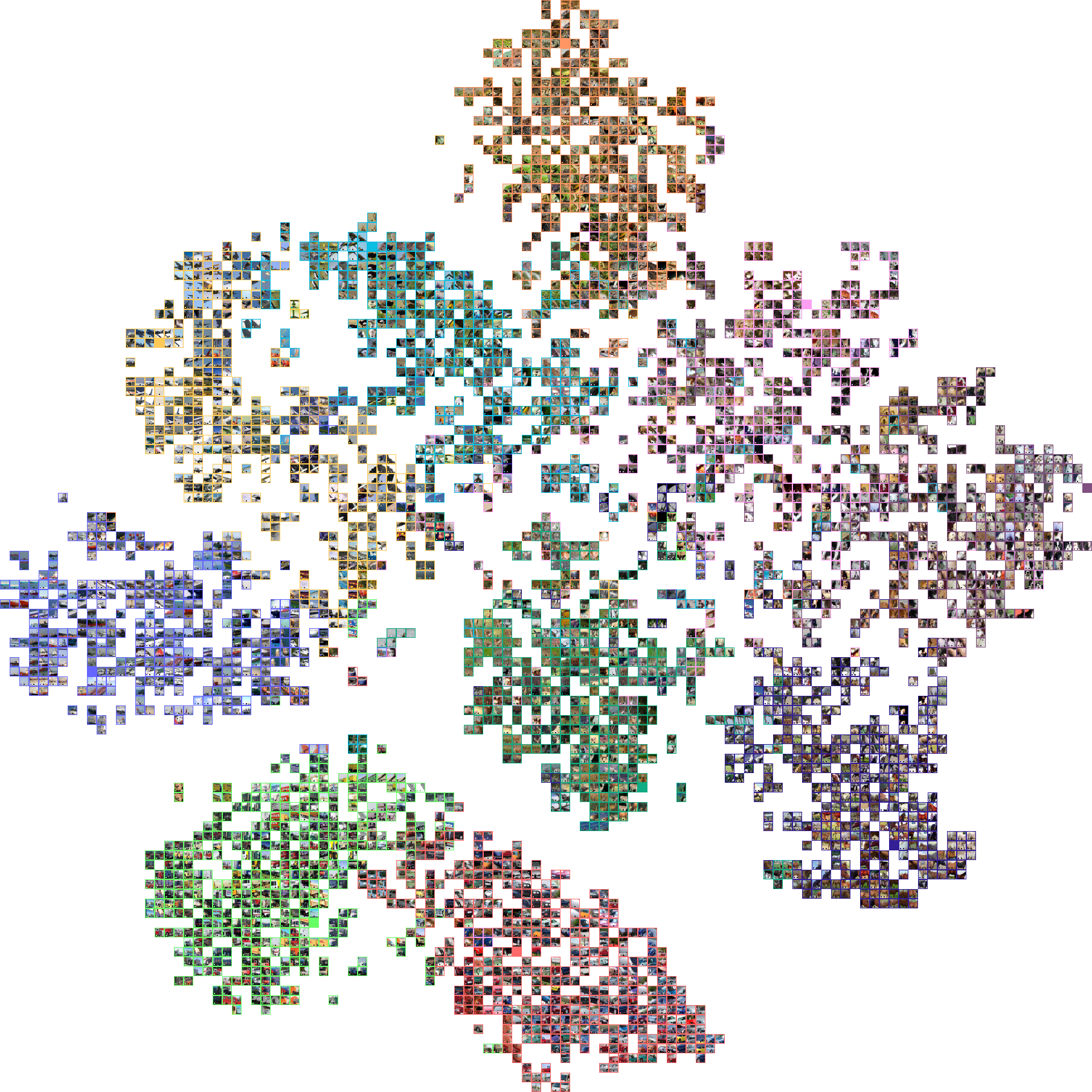}}
%  \vspace{1.5cm}
  %\centerline{\small{value}}
\end{minipage}
\begin{minipage}[b]{0.15\linewidth}
  \centering
  \centerline{\includegraphics[width=\linewidth]{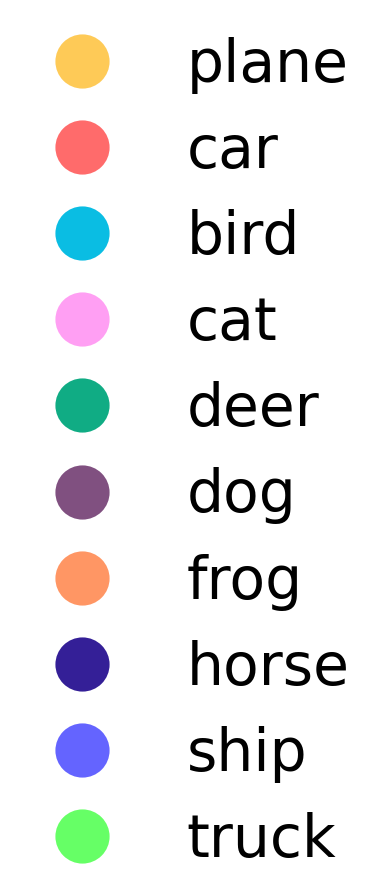}}
%  \vspace{1.5cm}
  %\centerline{\small{value}}
\end{minipage}
}
\caption{Patches embedded by 2D t-SNE with respect to their class predictions (left) and their embedding vectors (right). Magnified version is available in supplementary.}
\label{fig:embedding_geometry}
\end{wrapfigure}
\textbf{Embedding space geometry.} We first compare the geometry of the features just before the last stage. We provide the relevant 2D t-SNE projections in \cref{fig:geometry_comparison}. We observe that baseline RN's features are scattered across the space regardless of their higher level semantics. On the contrary, the features at the output of our block (\ie \emph{stage-4 input}) are clustered with respect to their semantics. In particular, \emph{animals} occupy the one half of the space whereas \emph{vehicles} lie in the other half. We further show that such behaviour is the result of \emph{value} vector embeddings. When we compare the features at the input and the output of our block (\ie \emph{stage-3 out} and \emph{stage-4 input}), we see that clustering occurs after our feature embedding, validating our mechanism of feature embedding by the matched semantics. That said, in \cref{fig:embedding_geometry}, we also plot the patches according to 2D t-SNE of their class predictions and the resultant embedding vectors as the weighted combination of the class \emph{value} vectors ($f^\prime$ in \cref{fig:method}). With class predictions, semantically similar patches are embedded apart (\eg \emph{car} and \emph{truck}). On the other hand, embedding vectors reshapes the geometry so that semantically similar entities are mapped close, yet another result supporting the effectiveness of \emph{feature embedding by template matching} mechanism.

% Below is an example of how to insert images. Delete the ``\vspace'' line,
% uncomment the preceding line ``\centerline...'' and replace ``imageX.ps''
% with a suitable PostScript file name.
% -------------------------------------------------------------------------
\begin{wrapfigure}[11]{l}[-5pt]{0.45\linewidth}
\vspace{0\intextsep}
  \centering
  \centerline{\includegraphics[width=1.0\linewidth]{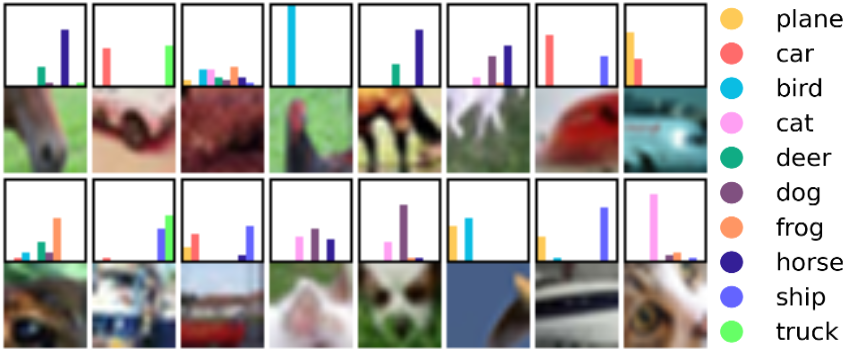}}
  \caption{Sample patches with their prediction scores stitched on the top.}
	\label{fig:words}
  \end{wrapfigure}

%
%\end{wrapfigure}

\textbf{Visual words}. To support our claims on generating vectors corresponding to new semantic entities from the combination of class vectors, we perform \emph{k-means} clustering with 100 centers of the class prediction scores. We then take the patches that are nearest to the centers. We provide 16 such patches in \cref{fig:words} together with their prediction scores. We observe that different combination of the classes means different semantic entities. For instance, \emph{wing} is generated by \emph{plane} and \emph{bird} classes, we have \emph{tire} as the combination of \emph{car} and \emph{truck}. We observe class-discriminative patches inheriting the class label. We as well observe more generic entities as the mixture of many classes such as \emph{fur} from \emph{animal} classes.

% CONCLUSION
% ----------
\section{Conclusion}
We reformulated convolution block based local feature embedding as feature assignment through best matching kernel and showed that $3\sxtimes 3$-$\mathrm{soft}$-$\mathrm{max}$-$1\sxtimes 1$ implements such a mechanism. Approximately relating $\mathrm{BN}$-$\mathrm{ReLU}$ to unnormalized soft-max, we brought a novel view point to $3\sxtimes 3$-$\mathrm{BN}$-$\mathrm{ReLU}$-$1\sxtimes 1$ which we encounter in popular ResNet-based models. Building on perspective explaining the bottom-up behavior of $3{\sxtimes}3$-BN-ReLU-$1{\sxtimes}1$ convolution block, we proposed a feature extraction mechanism that exploits weighted combination of class-semantic vectors to embed vector representation to the patches. We implemented such mechanism as a simple, yet effective residual layer. Our layer is learnable and effectively selects the classes that matches the patches most for feature embedding. We implemented our method with several architectures. With extensive empirical studies, we validated the effectiveness of our feature embedding layer as well as our theoretical claims.

% References
% -----------
\bibliography{library}

% Supplementary
% --------------
\appendix
% Title
% ----------
\section*{Supplementary Material for "\textit{Feature Embedding by Template Matching as a ResNet Block}"}

% Introduction
% -------------
\section{Architectural Details}
\label{sec:implementation}
\begin{table}[h]
	\centering
	\caption{Summary of the architectural choices including the input (in) and the output (out) feature dimensions, and the spatial reduction (reduction) of each stage for Cifar.}
	\label{tab:architectureCifar}
	\resizebox{\linewidth}{!}{%
    \begin{tabular}{lccccccccccccccc}
        \toprule
        \multicolumn{1}{l}{} & \multicolumn{3}{c|}{stage-1} & \multicolumn{3}{c|}{stage-2} & \multicolumn{3}{c|}{stage-3} & \multicolumn{3}{c|}{Our Block} & \multicolumn{3}{c}{stage-4} \\ \cmidrule{2-16}
        Architectures $\downarrow$ & \multicolumn{1}{c|}{in} & \multicolumn{1}{c|}{out} & \multicolumn{1}{c|}{reduction} & \multicolumn{1}{c|}{in} & \multicolumn{1}{c|}{out} & \multicolumn{1}{c|}{reduction} & \multicolumn{1}{c|}{in} & \multicolumn{1}{c|}{out} & \multicolumn{1}{c|}{reduction} & \multicolumn{1}{c|}{in} & \multicolumn{1}{c|}{out} & \multicolumn{1}{c|}{reduction} & \multicolumn{1}{c|}{in} & \multicolumn{1}{c}{out} & \multicolumn{1}{|c}{reduction}  \\ \midrule
        RN26 & 16 & 64 & $\downarrow$ 1 & 64 & 128 & $\downarrow$ 2 & 128 & 128 & $\downarrow$ 2 & - & - & - & 128 & 256 & $\downarrow$ 1 \\
        RN26-Ours & 16 & 64 & $\downarrow$ 1 & 64 & 128 & $\downarrow$ 2 & 128 & 128 & $\downarrow$ 2 & 128 & 128 & $\downarrow$ 1 & 128 & 256 & $\downarrow$ 1 \\ \midrule
        RN38 & 16 & 64 & $\downarrow$ 1 & 64 & 128 & $\downarrow$ 2 & 128 & 128 & $\downarrow$ 2 & - & - & - & 128 & 256 & $\downarrow$ 1 \\
        RN38-Ours & 16 & 64 & $\downarrow$ 1 & 64 & 128 & $\downarrow$ 2 & 128 & 128 & $\downarrow$ 2 & 128 & 128 & $\downarrow$ 1 & 128 & 256 & $\downarrow$ 1 \\ \midrule
        WRN16 & 16 & 32 & $\downarrow$ 1 & 32 & 64 & $\downarrow$ 2 & 64 & 128 & $\downarrow$ 2 & - & - & - & 128 & 128 & - \\
        WRN16-Ours & 16 & 32 & $\downarrow$ 1 & 32 & 64 & $\downarrow$ 2 & 64 & 128 & $\downarrow$ 2 & 128 & 128 & $\downarrow$ 1 & 128 & 128 & - \\ \midrule
        DN100 & 24 & 108 & $\downarrow$ 2 & 108 & 150 & $\downarrow$ 2 & 150 & 342 & - & - & - & - & 342 & 534 & - \\
        DN100-Ours & 24 & 108 & $\downarrow$ 2 & 108 & 150 & $\downarrow$ 2 & 150 & 342 & - & 342 & 342 & $\downarrow$ 1 & 342 & 534 & - \\
        DN100-Ours-C & 24 & 108 & $\downarrow$ 2 & 108 & 150 & $\downarrow$ 2 & 150 & 342 & - & 342 & 534 & - & 534 & 726 & - 
        \\\bottomrule
        \end{tabular}
    }
\end{table}
\begin{table}[h]
	\centering
	\caption{Summary of the architectural choices including the input (in) and the output (out) feature dimensions, and the spatial reduction (reduction) of each stage for Mini-Imagenet.}
	\label{tab:architectureImagenet}
	\resizebox{\linewidth}{!}{%
    \begin{tabular}{lccccccccccccccc}
        \toprule
        \multicolumn{1}{l}{} & \multicolumn{3}{c|}{stage-1} & \multicolumn{3}{c|}{stage-2} & \multicolumn{3}{c|}{stage-3} & \multicolumn{3}{c|}{Our Block} & \multicolumn{3}{c}{stage-4} \\ \cmidrule{2-16}
        Architectures $\downarrow$ & \multicolumn{1}{c|}{in} & \multicolumn{1}{c|}{out} & \multicolumn{1}{c|}{reduction} & \multicolumn{1}{c|}{in} & \multicolumn{1}{c|}{out} & \multicolumn{1}{c|}{reduction} & \multicolumn{1}{c|}{in} & \multicolumn{1}{c|}{out} & \multicolumn{1}{c|}{reduction} & \multicolumn{1}{c|}{in} & \multicolumn{1}{c|}{out} & \multicolumn{1}{c|}{reduction} & \multicolumn{1}{c|}{in} & \multicolumn{1}{c}{out} & \multicolumn{1}{|c}{reduction}  \\ \midrule
        RN26 & 16 & 64 & $\downarrow$ 2 & 64 & 128 & $\downarrow$ 2 & 128 & 128 & $\downarrow$ 2 & - & - & - & 128 & 256 & $\downarrow$ 1 \\
        RN26-Ours & 16 & 64 & $\downarrow$ 2 & 64 & 128 & $\downarrow$ 2 & 128 & 128 & $\downarrow$ 2 & 128 & 128 & $\downarrow$ 1 & 128 & 256 & $\downarrow$ 1 \\ \midrule
        RN38 & 16 & 64 & $\downarrow$ 2 & 64 & 128 & $\downarrow$ 2 & 128 & 128 & $\downarrow$ 2 & - & - & - & 128 & 256 & $\downarrow$ 1 \\
        RN38-Ours & 16 & 64 & $\downarrow$ 2 & 64 & 128 & $\downarrow$ 2 & 128 & 128 & $\downarrow$ 2 & 128 & 128 & $\downarrow$ 1 & 128 & 256 & $\downarrow$ 1 \\ \midrule
        WRN16 & 16 & 32 & $\downarrow$ 2 & 32 & 64 & $\downarrow$ 2 & 64 & 128 & $\downarrow$ 2 & - & - & - & 128 & 128 & - \\
        WRN16-Ours & 16 & 32 & $\downarrow$ 2 & 32 & 64 & $\downarrow$ 2 & 64 & 128 & $\downarrow$ 2 & 128 & 128 & $\downarrow$ 1 & 128 & 128 & - \\ \midrule
        DN100 & 24 & 108 & $\downarrow$ 2 & 108 & 150 & $\downarrow$ 2 & 150 & 171 & $\downarrow$ 2 & - & - & - & 171 & 363 & - \\
        DN100-Ours & 24 & 108 & $\downarrow$ 2 & 108 & 150 & $\downarrow$ 2 & 150 & 171 & $\downarrow$ 2 & 171 & 171 & $\downarrow$ 1 & 171 & 363 & - \\
        DN100-Ours-C & 24 & 108 & $\downarrow$ 2 & 108 & 150 & $\downarrow$ 2 & 150 & 171 & $\downarrow$ 2 & 171 & 363 & - & 363 & 555 & - 
        \\\bottomrule
        \end{tabular}
    }
\end{table}

We provide details of the architectural choices for the baseline methods for the sake of reproducibility of our experimental work. We use ResNet (RN) \cite{he2016identity}, Wide-ResNet (WRN) \cite{zagoruyko2016wide}, and DenseNet (DN) \cite{huang2017densely} as the baseline architectures. We use 4 \emph{stages} for each architecture. Note that the implementation of the \emph{stage} differs from method to method as we will disclose shortly. %Also, we provide some implementation tricks for all models while adapting our method for each dataset separately.  

\textbf{ResNet (RN).} We stick to the original implementation of ResNet v2 \cite{he2016identity} including the combination of convolution, batch normalization (BN) and ReLU layers at the start of the first stage. In our notation, a typical RN v2 \emph{stage} includes multiple residual blocks which are called \emph{bottleneck residual units}. The first block of each stage perform  $1{\sxtimes}1$ convolution in the shortcut connection. For the stages that perform spatial reduction, the stride of that convolution is 2. We use an additional stage (stage-4) to incorporate our method easily during implementation. We perform experiments with two RN architectures with the number of blocks for each stage being 2 (RN26) and 3 (RN38), respectively. We summarize the architecture details in \cref{tab:architectureCifar,tab:architectureImagenet} for Cifar \cite{krizhevsky2009learning} and Mini-Imagenet \cite{Ravi2017OptimizationAA}, respectively. $\downarrow k$ in reduction means we have $1{\sxtimes}1$ convolution with stride $k$ in the shortcut connection before addition, and $-$ means direct shortcut connection. Only for Cifar 10, we find that using an additional $2{\sxtimes}2$ average pooling in the shortcut before the $1{\sxtimes}1$ convolution layer (\emph{i.e.}, linear transform) better generalizes the incoming features. Moreover, we use the output of the BN as the input to the \emph{soft-max} operation to shape the softness of the \emph{soft-max} predictions. With that being said, one can use temperature scaling to logits instead. Yet BN performs such a scaling inherently since it provides us with scaled and normalized activations. Hence we do not have to choose the temperature manually. Such tricks bring marginal improvements to the performance in Cifar 10.

\textbf{Wide-ResNet (WRN).} We stick to the original implementation of WRN \cite{zagoruyko2016wide} including a single convolution layer at the start of the first stage. Similar to ResNet, a typical WRN \emph{stage} includes multiple residual blocks which are called \emph{basic residual architecture} in the original paper \cite{zagoruyko2016wide}. For the stages that perform spatial reduction, the first block includes $1{\sxtimes}1$ convolution with stride 2 in the shortcut connection. If the channel dimensions of the input and the output features are not the same for that stage, the first block again includes $1{\sxtimes}1$ convolution with stride 1 in the shortcut connection. We use an additional stage (stage-4) to incorporate our method easily during implementation. We use WRN of \emph{depth} 16 and \emph{widening factor} 2. Namely, the depth of the each stage is computed so that the total depth is 16. We do not use \emph{dropout}. We summarize the architecture details in \cref{tab:architectureCifar,tab:architectureImagenet} for Cifar and Mini-Imagenet, respectively. $\downarrow k$ in reduction means we have $1{\sxtimes}1$ convolution with stride $k$ in the shortcut connection before addition, and $-$ means direct shortcut connection. In our block, we use an extra BN before $4{\sxtimes}4$ average pooling owing to the slight architectural differences of WRN from RN (In fact, we are doing that to make the internal classification stage of the patches more similar to the final classification stage of the original network). Only for Cifar 10, we use the same implementation tricks as in RN.

\textbf{DenseNet (DN).} We stick to the original implementation of DN-BC \cite{huang2017densely} including a single convolution layer at the start of the first stage. In the context of DN, a typical \emph{stage} includes a multiple-layered \emph{dense block} \cite{huang2017densely}, and a \emph{transition layer} \cite{huang2017densely} if reduction is specified. We use \emph{bottleneck implementation} \cite{huang2017densely} in dense blocks with 0.5 \emph{compression factor} \cite{huang2017densely}  at the transition layers since we are using DN-BC. The \emph{compression factor} reduces the channel dimension by the specified factor. We use an additional stage (stage-4) to incorporate our method easily during implementation. We use DN of \emph{depth} 100 and \emph{growth rate} 12. Namely, the depth of the each stage is computed so that the total depth is 100. For our method, we additionally employ concatenation of the embedded feature and the input feature instead of addition through shortcut to align with the architectural design of DN, which is referred as DN100-Ours-C. For DN100-Ours-C, we find that using value vectors of the half dimension of the input gives good results. Thus, we use 192 dimensional \emph{value} vectors (\ie 192-many $1{\sxtimes}1$ convolutions for embedding) and concatenate them with the corresponding input. Aligned with the baseline architecture, we do not use $1\sxtimes 1$ convolution in the shortcut connection.
We summarize the architecture details in \cref{tab:architectureCifar,tab:architectureImagenet} for Cifar and Mini-Imagenet, respectively. $\downarrow 2$ in reduction means we have the \emph{transition layer} in between DN stages, $\downarrow 1$ means we have $1{\sxtimes}1$ convolution in our shortcut connection, and $-$ means direct connection without any convolution.  Similar to WRN, we use an extra BN before $4{\sxtimes}4$ average pooling in our block. Only for Cifar 10, we use the same implementation tricks as in RN and WRN except that we do not use $2{\sxtimes}2$ average pooling in the shortcut since it results in over smoothing considering the transition layers also inheriting $2{\sxtimes}2$ average pooling.

%We provide the summary for the input and the output feature dimensions for every architecture with or without our block considering Cifar (10 and 100) and 100-class Mini-Imagenet in \cref{tab:architectureCifar} and \cref{tab:architectureImagenet}, respectively. We include the \emph{reduction} parameter as a substitute for the inclusion of \emph{stride} 2 in the first residual block in RN and WRN, and a dense layer used in the shortcut connection for our block. $\downarrow$ 2 shows that there exist a reduction in the resolution at that stage. $\downarrow$ 1 shows that there is no reduction in the resolution at that stage, yet a linear transformation (\emph{e.g.}, a transition layer, a $1{\sxtimes}1$ convolution layer) still exists. If nothing is stated, it means that there is no additional layer is used.    

%As we can see from \cref{tab:architectureCifar} and \cref{tab:architectureImagenet}, our block does not change the input dimension for the last stage compared with the baseline except for DN100-Ours-C. With DN100-Ours-C, in our block, we arrange the dimension of the class embedding, \emph{i.e., value}, vectors as the half of the input dimension and concatenate it with the input. This results in different output dimension for our block, further affecting the last stage dimensions as well. 

\section{Magnified Figures and Discussion}
\label{sec:analysis}

We provide magnified visualizations of class predictions and embedding vectors of patches in \cref{fig:softmax,fig:value}, a summary of which is already included in the main paper. Specifically, we generate a sprite image of the patches, where each patch is embedded with respect to its class prediction vector (\cref{fig:softmax}) or embedding vector as the convex combination of class embedding, \ie \emph{value}, vectors (\cref{fig:value}). We enhance patch images with further visual aids as illustrated in  \cref{fig:convention}, in which the color in the frame represents the true class, the colored box in the left corner represents the final predicted class coming from the classifier and the grayish-filled box in the middle represents the entropy calculated from the \emph{soft-max} predictions of our block for each patch image extracted. We especially use entropy calculation to understand how peaky or how uniform the \emph{soft-max} predictions are to further interpret the results. The color of the entropy indicator box goes darker as the entropy goes lower and vice versa. 

% Below is an example of how to insert images. Delete the ``\vspace'' line,
% uncomment the preceding line ``\centerline...'' and replace ``imageX.ps''
% with a suitable PostScript file name.
% -------------------------------------------------------------------------
\begin{figure}[!ht]%[15]{r}[-5pt]{0.8\linewidth}
%\vspace{-1\intextsep}
  \centering
  \centerline{\includegraphics[width=1.0\linewidth]{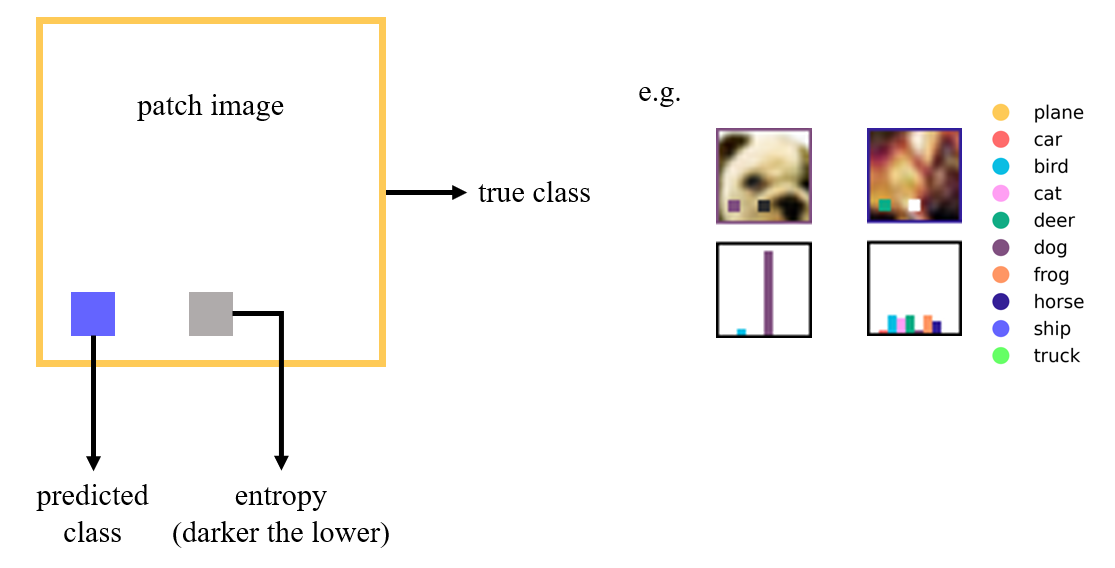}}
  \caption{The convention used in the visualization of 2D t-SNE projections of the features.}
	\label{fig:convention}
  \end{figure}

%
%\end{wrapfigure}

Showcasing the patch image convention, we include two examples in \cref{fig:convention} with their corresponding histograms obtained from the \emph{soft-max} predictions of our block, which are also used in the entropy calculation. These examples consist of one patch image having a \emph{dog face} and the other one having an \emph{animal body}. Taking the \emph{dog} image for instance, our \emph{soft-max} predictions (histogram) have a peak at \emph{dog} class and the final classifier (the box in the left corner) as well predicts the class \emph{dog}, which can be seen from its color. This color also matches with the frame color, indicating that we make the correct assignment for the image from which the patch is extracted. Moreover, since the histogram is very peaky, we have a very low entropy. Hence, we have a darker colored box in the middle as expected. As an example of another case, we predict the wrong class in the final classifier for the image from which the patch with an \emph{animal body} is extracted. Note that in that case the color of the frame and the box in the left corner do not match. We also have a relatively higher entropy which is indicated by the brightness of the middle box. We indeed expect such kind of results for the patches having semantic entities which are shared among the classes. That being said, we see from the corresponding histogram that the non-zero histogram bins only come from \emph{animal} classes as expected. Due to the structure of the body and the combination of the other corresponding patch images at the final classification stage, the final prediction becomes the class \emph{deer} instead of \emph{horse}, which are close species in nature.

% Using the aforementioned conventions for the illustrations, we provide two 2D t-SNE plots for the \emph{soft-max} predictions and the value embedding vectors of our block in \cref{fig:softmax} and \cref{fig:value}, respectively. %Although in \cref{fig:softmax}, classes such as truck and car, or bird and plane are embedded apart, with our \emph{feature embedding by template matching} mechanism, our value embedding vectors reshape the geometry so that semantically similar entities are embedded closer to each other, creating a more cluster-like structure.

For \cref{fig:value}, we additionally embed the class value vectors as the images filled with solid colors corresponding to classes. We also embed 0-$vector$ as a black filled image. We observe that value vectors can be considered as the vertices of the convex hull of the embedded features. Hence, their convex combination creates the corresponding embedding vectors.

Once we look at the origin (black box) by zooming in the image, we see that the patch images nearby have larger entropy compared to the ones away from the 0-$vector$. In other words, the patches of high entropy are assigned to 0-$vector$. Such behavior is not surprising since we believe the patches of high entropy (\ie shared among many classes) should not carry too much information. These patches generally include shared nuisance information than discriminative patterns such as \emph{beak}, \emph{ear} and \emph{wing}. Similarly, we observe relatively higher entropy of the predictions in the transition between classes such as \emph{bird} and \emph{plane}, \emph{car} and \emph{truck} or \emph{plane} and \emph{ship}. These passage points represent mutual semantic entities for those classes, such as \emph{wing} for \emph{bird} and \emph{plane}, \emph{tire} for \emph{car} and \emph{truck} or \emph{blue background} for \emph{plane} and \emph{ship}, yet another supporting result for our claims on combining class labels to generate novel labels corresponding to different semantic entities. For the discriminative entities (\ie the ones nearby the class value vectors), the distinction between the classes are more clear, resulting in smaller entropy.

%Yet, with the help of the final stage, the combination of all patches for the same image are transmitted to the classifier and mostly, the final predictions matches with the true class.
%This behavior can also be seen in patches with high entropy and patches clustered in the wrong class. This shows the power and the effectiveness of our hierarchical network where we perform coarse-to-fine grained classification. 

\newpage
% Below is an example of how to insert images. Delete the ``\vspace'' line,
% uncomment the preceding line ``\centerline...'' and replace ``imageX.ps''
% with a suitable PostScript file name.
% -------------------------------------------------------------------------
\begin{figure}[!ht]
\begin{minipage}[b]{0.91\linewidth}
  \centering
  \centerline{\includegraphics[width=\linewidth]{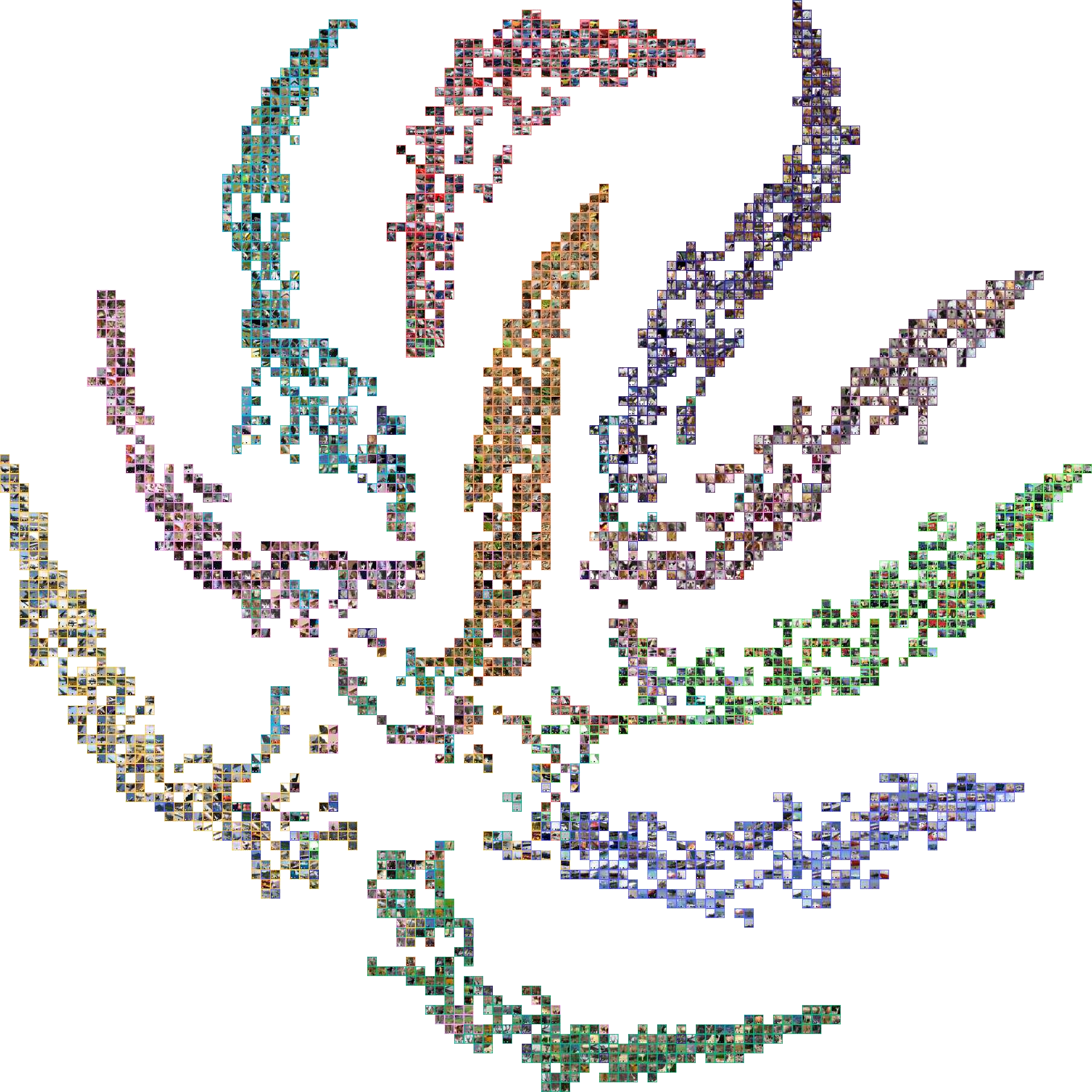}}
%  \vspace{2.0cm}
\end{minipage}
\begin{minipage}[b]{.08\linewidth}
  \centering
  \centerline{\includegraphics[angle=0,origin=c, width=\linewidth]{figures/legend_final.png}}
%  \vspace{1.5cm}
\end{minipage}

  \caption{Patches embedded by 2D t-SNE with respect to their class predictions.}
	\label{fig:softmax}
  \end{figure}
  
\newpage

\begin{figure}[!ht]
    
    \begin{minipage}[b]{0.91\linewidth}
      \centering
      \centerline{\includegraphics[width=\linewidth]{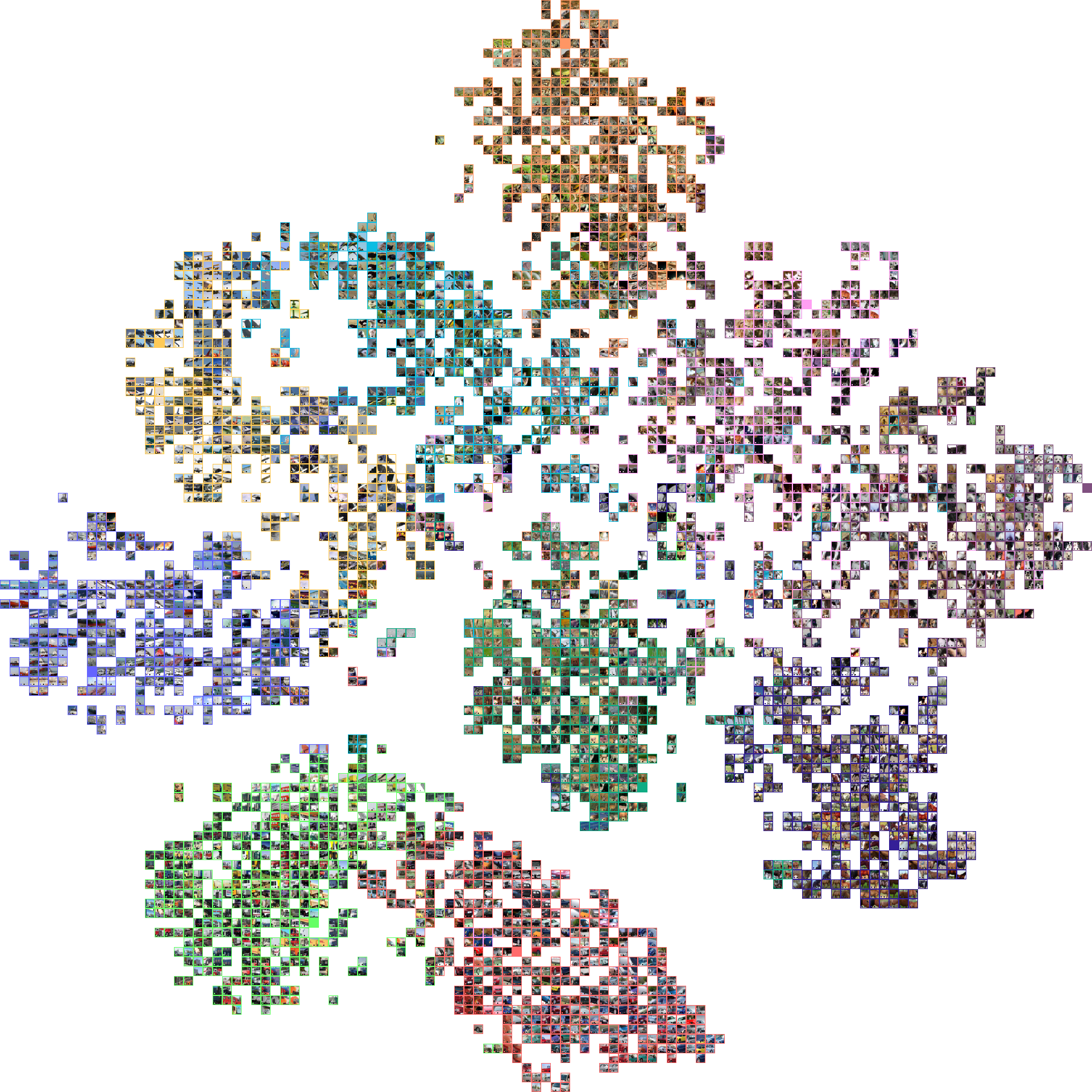}}
    %  \vspace{2.0cm}
    \end{minipage}
    \begin{minipage}[b]{.08\linewidth}
      \centering
      \centerline{\includegraphics[angle=0,origin=c, width=\linewidth]{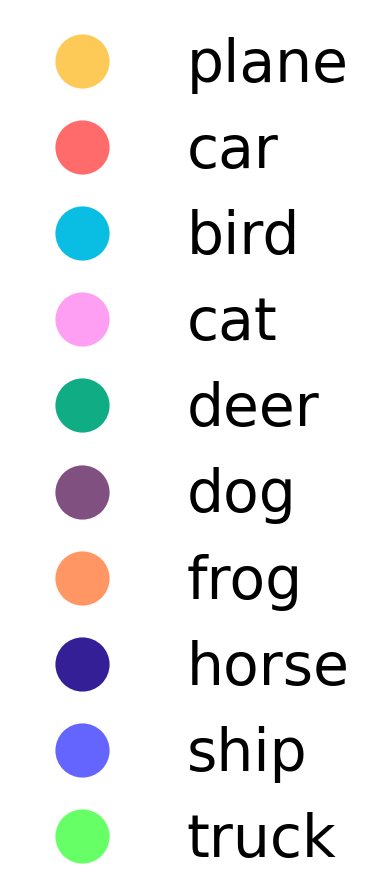}}
    %  \vspace{1.5cm}
    \end{minipage}

  \caption{Patches embedded by 2D t-SNE with respect to convex combination of class embedding vectors.}
	\label{fig:value}
  \end{figure}
%
%\end{wrapfigure}

\end{document}